\documentclass[11pt]{article}
\usepackage[utf8]{inputenc}
\usepackage{microtype}
\usepackage{amsmath}
\usepackage{amsthm}
\allowdisplaybreaks
\usepackage{amssymb}
\usepackage{subfig}
\usepackage{color}
\usepackage{epstopdf}
\usepackage{url}
\usepackage{graphicx}
\usepackage{color}
\usepackage{epstopdf}
\usepackage{algpseudocode}
\usepackage{scrextend}
\usepackage[T1]{fontenc}
\usepackage{bbm}
\usepackage{comment}
\usepackage{booktabs}
\usepackage{multirow}
\usepackage[table,xcdraw]{xcolor}
\usepackage{changepage,threeparttable} 

\usepackage{hyperref}
\hypersetup{colorlinks=true,citecolor=red,linkcolor=red}

\usepackage{natbib}
\usepackage[margin=1in]{geometry}

\usepackage[noend,boxruled,linesnumbered]{algorithm2e}
\usepackage{setspace}
\SetKwInput{KwInput}{Input}                
\SetKwInput{KwOutput}{Output}              
\SetKwInput{KwTarget}{Target}
\SetKwInput{KwPar}{Parameters} 

\newtheorem{theorem}{Theorem}[section]
\newtheorem{lemma}[theorem]{Lemma}
\newtheorem{definition}[theorem]{Definition}

\DeclareMathOperator*{\E}{{\mathbb{E}}}

\newcommand{\lp}{$P_L$}
\newcommand{\np}{$P_N$}
\newcommand{\trans}[1]{\mathsf{tran}(#1)}

\title{Differentially Private Label Protection in Split Learning }
\author{
Xin Yang\thanks{Correspondence to \texttt{yangxin.yx@bytedance.com}}\\\texttt{yangxin.yx@bytedance.com}
\and Jiankai Sun \\\texttt{jiankai.sun@bytedance.com}
\and Yuanshun Yao \\ \texttt{kevin.yao@bytedance.com}
\and Junyuan Xie \\ \texttt{junyuan.xie@bytedance.com}
\and Chong Wang \\ \texttt{chong.wang@bytedance.com}
}

\begin{document}

\maketitle
\begin{abstract}
    Split learning is a distributed training framework that allows multiple parties to jointly train a machine learning model over vertically partitioned data (partitioned by attributes). 
The idea is that only intermediate computation results, rather than private features and labels, are shared between parties so that raw training data remains private.
Nevertheless, recent works showed that the plaintext implementation of split learning suffers from severe privacy risks that a semi-honest adversary can easily reconstruct labels. 
In this work, we propose \textsf{TPSL} (Transcript Private Split Learning), a generic gradient perturbation based split learning framework that provides provable differential privacy guarantee.
Differential privacy is enforced on not only the model weights, but also the communicated messages in the distributed computation setting.
Our experiments on large-scale real-world datasets demonstrate the robustness and effectiveness of \textsf{TPSL} against label leakage attacks. We also find that \textsf{TPSL} have a better utility-privacy trade-off than baselines. 


\end{abstract}

\section{Introduction}

In the past decade, deep learning has achieved tremendous success in many applications including recommendation systems, autonomous driving and healthcare. 
Nevertheless, the abundant data needed to train deep learning models draws increasing concerns over user privacy,
and strict regulations like CCPA\footnote{California Consumer Privacy Act}, HIPAA\footnote{Health Insurance Portability and Accountability Act} and GDPR\footnote{General Data Protection Regulation, European Union} have been introduced to regulate how data can be transmitted and used.
To address privacy concerns,
\textit{Federated Learning} (FL)~\citep{mmr+17,hhh+20,ym20,gcyr20}
is proposed so that multiple entities can collaborate to train a model without sharing raw data.

One important category of FL is \emph{vertical Federated Learning} or vFL~\citep{ylct19}, where each party has training data of different subjects, but the attributes (e.g. features and labels) of the  data are the same.
vFL arises in many scenarios like finance and healthcare~\citep{ylct19,lsy+21}.
Some vFL frameworks apply secure multi-party computation (MPC) protocols including oblivious transfer, garbled circuits~\citep{gsb+16} and Homomorphic Encryption (HE)~\citep{hhi+17,yfcsy19,ylct19} to protect data privacy. However, this will cause significant overhead in computation and communication~\citep{ksw+18}.
Hence such methods only works with simple models like linear regression and logistic regression.

For larger neural network models,
\citet{gr18} proposed the idea of \emph{split neural network}, or split learning. This is also the focus of this paper.
The goal of split learning is to train a neural network with the standard centralized training algorithms like mini-batch stochastic gradient descent (SGD) in an MPC setting, while no raw data is shared between parties.
To do so,
a neural network is split into multiple sub-networks, and each party holds one such sub-network. 
All the parties simulate centralized training as follows:
each party collects  input data from either other pieces or raw input owned by the party itself,
then they do the local computation within the piece,
and finally they send the output (e.g. embeddings or gradients) to other parties that need those intermediate results. Compared to secure MPC methods, split learning is straightforward to implement, and it has small computational/communication cost that is similar to centralized training.


The security of split learning is based on the assumption  that neural networks behave like black-boxes and intermediate computation results are so complicated to parse that attackers cannot use them to recover the raw training data.
Nevertheless,
recently \citet{lsy+21} showed that the adversarial party can recover private labels with high accuracy by computing the norm and direction of the gradients. This is not totally unexpected. After all, there is no provable privacy guarantee for split learning. In order to apply split learning in practice, we need to build the learning algorithms on theoretical foundations.


\emph{Differential privacy} (DP)~\citep{dmns06,dr14} is a widely applied measurement of privacy and has been deployed in many real-world applications~\citep{dky17,a18}. DP has become the standard practice in privacy-preserving machine learning~\citep{cms11,l11,acg+16}.
It is natural to ask if we can bring DP into the split learning setting. However, one major difference of split learning (or the more general vFL setting) between traditional machine learning algorithms is that the parties in split learning need to exchange messages (transcripts) and hence more information is exposed besides the model parameters. This is a new challenge for designing privacy-preserving learning protocols.



In this work,
we present Transcript Private Split Learning, or \textsf{TPSL},  a novel framework of split learning with strong provable DP guarantee.
In particular,
we focus on protecting label information with DP.
This is because not only label leakage is a known privacy threat in split learning~\citep{lsy+21}, but also there are many cases where  labels are sensitive and require protection, like online advertising and recommendation systems~\citep{ggk+21}.


The security of \textsf{TPSL} is provided by our rigorous and complete proof on DP guarantee.
The computation and communication cost of \textsf{TPSL} is similar to the naive implementation of split learning, as differential privacy can be efficiently enforced by noise perturbation mechanisms.
Inspired by the white-box attack of split learning~\citep{ekc21}, 
\textsf{TPSL} adds suitable noise to the gradients so that the given privacy budget is met.
Compared to the baseline of adding isotropic noise on the transcripts,
\textsf{TPSL} deals with the perturbation more carefully and efficiently. 
Experiments over large-scale data verify the effectiveness of \textsf{TPSL} to achieve better utility-privacy trade-off.

\section{Background}

\subsection{Split Learning}
Split learning~\citep{vgsr18,gr18,akk+20,csm+20} is a distributed deep learning framework that allows multiple parties to jointly train a neural network without directly sharing raw data. In this work, we focus on two-party split learning over binary classification problem, where the domain of features is $\mathcal{X}$ and the label is in $\mathcal{Y}=\{0,1\}$.
For simplicity of exposition we consider the scenario where the training data is perfectly vertically split.
Following the notations in \cite{lsy+21}, 
{\lp} denotes the label party that holds all the labels and {\np} denotes the non-label party that holds all the features.
We assume that at the beginning of split learning, features and labels have already been aligned by protocols such as private set intersection~\citep{kkrt16,psz18}.

The goal of the two parties is to jointly train a prediction model $h_{\theta_L}\circ f_{\theta_N}$.
Here $f:\mathcal{X}\rightarrow \mathbb{R}^{d}$ is the feature representation function held by \np, $\theta_N$ is the model parameters of \np{} and $d$ is the dimension of the intermediate representation (embeddings).
Similarly, $h:\mathbb{R}^d\rightarrow [0,1]$ is the prediction function held by \lp{} and $\theta_{L}$ is the model parameters of \lp.

In this paper, we consider $h,f$ are neural networks with nonlinear activation functions, such as ReLU. We consider the standard mini-batch stochastic gradient descent (SGD) training with loss function $\ell$.
For the ease of exposition,
we consider an online fashion of training in the sense that each training sample is used once, which is presented in Algorithm~\ref{alg:sgd}.
On a high level, in the forward phase,
\np{} computes the embeddings $f(\mathbf{x})$ and passes them to \lp;
\lp{} continues to compute the prediction and loss.
In the backward phase, \lp{} computes the gradients $g=\frac{\partial \ell(h(f(\mathbf{x}),y))}{\partial f(\mathbf{x})}$ and return them to \np.
Finally, \lp{} and \np{} update their models with the gradients.

\setlength{\algomargin}{15pt}
\begin{algorithm*}[!ht]
 \caption{Stochastic Gradient Descent in Split Learning}\label{alg:sgd} 
\KwIn{Features 
$X$ held by \np. 
Labels $Y$ 
held by \lp. $|X|=|Y|=n$.}
\KwPar{Number of batches $B$. Step size $\eta$. Loss function $\ell:[0,1]\times \{0,1\}\rightarrow \mathbb{R}_{\geq 0}$.}
Training samples are divided into $B$ batches.\\
$\theta_N,\theta_L$ are initialized with Gaussian initialization.\\
\For{$b=1,2,\cdots,B$}{
\tcp{Forward phase}
\np{} computes embeddings $e_{b,i}:=f_{\theta_N}(\mathbf{x}_{b,i})$ for $i=1,\cdots,n/B$ where $\mathbf{x}_{b,1},\mathbf{x}_{b,2},\cdots, \mathbf{x}_{b,n/B}$ are the features of samples in batch $b$.\\
\np{} sends $\{e_{b,1},\cdots,e_{b,n/B}\}$ to \lp.\\
\tcp{Backward phase}
\lp{} computes the prediction of the neural network $h$ as $\tilde y_{b,i} := h_{\theta_L}(f_{\theta_N}(\mathbf{x}_{b,N}))$ for $i=1,\cdots,n/B$.\\
\lp{} computes losses $\ell_{b,i} := \ell(\tilde y_{b,i}, y_{b,i})$ for $i=1,\cdots,n/B$ where $y_{b,1},y_{b_2},\cdots,y_{b,n/B}$ ae the labels in batch $b$.\\
\lp{} computes gradients $g_{b,i} := \frac{\partial \ell_{b,i}}{\partial e_{b,i}}$ for $i=1,\cdots,n/B$.\\
\lp{} sends $\{g_{b,1},\cdots,g_{b,n/B}\}$ to \np.\\
\tcp{Model updates}
\lp{} updates $\theta_L$ by $\theta_L\leftarrow \theta_L- \eta\cdot \frac{B}{n}\cdot \sum_{i=1}^{n/B}\frac{\partial \ell_{b,i}}{\partial \theta_L}$.\\
\np{} computes $\frac{\partial \ell_{b,i}}{\partial \theta_N} = g_i^\top \frac{\partial e_{b,i}}{\theta_N}$ and then  updates $\theta_N$ by $\theta_N\leftarrow \theta_N- \eta\cdot \frac{B}{n}\cdot \sum_{i=1}^{n/B}\frac{\partial \ell_{b,i}}{\partial \theta_N}$.\\
}
\end{algorithm*}

\subsection{Differential Privacy}

Differential privacy (DP)~\citep{dmns06,dr14} is a quantifiable and rigorous privacy framework.
We adopt the following definition of DP.
\begin{definition}[Differential Privacy]
Fix $\epsilon,\delta\in \mathbb{R}_{\geq 0}$.
We say that a randomized mechanism $\mathcal{M}$ is $(\epsilon,\delta)$-differentially private, if for any neighboring datasets $D,D'$,
and for any subset $S$ of possible output of $\mathcal{M}$,
\begin{align*}
    \Pr[\mathcal{M}(D)\in S]\leq e^{\epsilon}\cdot \Pr[\mathcal{M}(D')\in S]+\delta.
\end{align*}
\end{definition}

In split learning, which is a type of MPC, we would like to protect privacy of information shared between the parties.
For an MPC protocol $A$,
we denote $\trans{A}$ as the \textit{transcript} of $A$, 
namely the messages shared between the parties.
In split learning,
the transcript includes (1) embeddings in the forward phase $f(\mathbf{x})$ and (2) gradients in the backward phase $g=\frac{\partial \ell(h(f(\mathbf{x}),y))}{\partial f(\mathbf{x})}$. 
We can use DP to measure the privacy of information exchange in MPC protocols,
which we call \textit{transcript differential privacy}: 
\begin{definition}[Transcript Differential Privacy]
Fix $\epsilon,\delta\in \mathbb{R}_{\geq 0}$.
We say that a randomized mechanism $A$ between multiple parties is $(\epsilon,\delta)$-transcript differentially private,
if $\trans{A}$ is $(\epsilon,\delta)$-differentially private.
\end{definition}

The notion of neighboring datasets determines the privacy that DP aims to protect. As we focus on label protection, we say two datasets $D,D'$ are neighbouring if they only differ at one label in the samples.
This definition of neighbouring dataset is similar to the setting of label DP~\citep{ggk+21,mmp+21}.

\textit{Sensitivity} in differential privacy is defined as the following:
Fix $p\geq 0$,
the $\ell_p$ sensitivity of function $f$ is $\max_{D,D'}\|f(D)-f(D')\|_p$ where the maximization is over neighbouring dataset $D,D'$.



\subsection{Threat Model}


As the defender, the label party \lp{} wants to protect the labels from the attacker (  non-label party \np{}). 
We consider  the attacker is  semi-honest who faithfully executes the split learning protocol (e.g. no adversarial chosen inputs), but it aims to infer the labels held by \lp{} from the information received during the protocol execution. In terms of attacker's capability, we consider two threat models: \textit{black-box} attack and \textit{white-box} attack. In black-box setting, the attacker only has access to its own inputs (features set $X$ and model parameters $\theta_N$) and the transcript $\trans{A}$. As shown by~\citet{lsy+21}, the attacker can still infer the labels based on gradient information under the black-box setting. In white-box setting, the attacker can also access $\theta_L$. In our design, we leverage insights coming from the white-box attack because it is the strongest attack assumption that should be considered in order to design a strong defense.  


\subsection{Notations} For positive integers $m<n$, $[n]$ represents the set $\{1,2,\cdots,n\}$,
and $[m,n]$ represents the set $\{m,m+1,\cdots,n\}$.
Let $\mathcal{N}(\theta,\Sigma)$ denote the (multi-variate) normal distribution with mean $\theta$ and covariance $\Sigma$.
For a vector $v$ and a set $S$ that is a subset of the underlying coordinates of $v$, $[v]_S$ represents the restriction of $v$ on $S$, i.e., if $S=\{s_1,s_2,\cdots,s_t\}$,
then $[v]_S=(v_{s_1},v_{s_2},\cdots,v_{s_t})$. 

\section{Design of Differentially Private Noise}

\subsection{Gaussian mechanism under white-box attack}
\label{sec:white_box_attack}
Gaussian mechanism~\citep{dr14} is a common post-hoc mechanism to implement DP. 
Let $f$ be a deterministic function with range $\mathbb{R}^m$,
and $\Delta_f$ be the $\ell_2$ sensitivity of $f$.
To achieve DP, Gaussian mechanism adds a Gaussian noise $R\in \mathbb{R}^m$ to the output of $f$.
It is shown that when $R\sim \mathcal{N}(0,\Delta_f^2\sigma^2\cdot I_m)$ with $\sigma \geq \frac{\sqrt{2\ln(1.25/\delta)}}{\epsilon}$,
the randomized function $f^{dp}(D):=f(D)+R$ is $(\epsilon,\delta)$-DP.

To ensure $\trans{A}$ to be DP,
a natural choice is to apply Gaussian mechanism on $\trans{A}$, and for the label party \lp{} who wants to protect its label, it is reasonable to add noise to the part of the $\trans{A}$ that \lp{} owns, namely the gradients. 
This approach is also one of our baselines in the experimental section. 



To improve from the vanilla Gaussian mechanism, we consider what the adversary might do when facing the basic Gaussian mechanism protection under the strong white-box capability. 
This helps us design the protection method against that strong white-box attack, which can bring insights on designing transcript DP split learning protocols.

Fix a training sample $(\mathbf{x},y)$, the adversarial \np{} knows the feature $\mathbf{x}$ and its forwarded embedding $f_{\theta_N}(\mathbf{x})$. It also receives a perturbed gradient $\tilde g:=g_y+R$ after \lp's DP protection ,
where $g_y\in \mathbb{R}^d$ is the correct gradient with respect to the label $y$ held by \lp, and $R\sim \mathcal{N}(0,\sigma^2 I_d)$ is the Gaussian noise added with suitable variance $\sigma^2$.



Under the white-box assumption, when \np{} knows $\theta_L$, it can perform the following attack:
\np{} can guess the $y$ that corresponds to the gradient with shorter distance to the perturbed gradient~\citep{ekc21}.
Formally, \np{} can first compute the gradients with respect to label $0,1$ by
\begin{align*}
    g_0:=&~\frac{\partial \ell(h_{\theta_L}(f_{\theta_N}(\mathbf{x})),0)}{\partial f_{\theta_N}(\mathbf{x})},\\
    g_1:=&~\frac{\partial \ell(h_{\theta_L}(f_{\theta_N}(\mathbf{x})),1)}{\partial f_{\theta_N}(\mathbf{x})}.
\end{align*}
Then shortest distance attack guesses the value of $y$ as
\begin{align*}
    \bar y :=
    \begin{cases}
        0  & \mbox{ If }\|\tilde g - g_0\|_2 \leq \|\tilde g - g_1\|_2,  \\
        1  & \text{ Else.}
    \end{cases}
\end{align*}

We name this attack  \textit{shortest distance attack}. 
If there is no protection, then shortest distance attack can obtain the correct label $y$ almost surely. 
Now we explain how we can leverage shortest distance attack to design an improved mechanism.
Notice that the decision of shortest distance attack is merely determined by the difference $\|\tilde g - g_{0}\|_2^2-\|\tilde g - g_{1}\|_2^2$.
By expanding this difference, we have
\begin{align*}
    &~\|\tilde g - g_{0}\|_2^2-\|\tilde g - g_{1}\|_2^2\\
    = & ~\langle g_1-g_0,2\tilde g - g_0-g_1\rangle\\
    = & ~\langle g_1-g_0,2(g_y+R) - g_y-g_{1-y}\rangle\\
    = & ~\langle g_1-g_0,g_y-g_{1-y}\rangle+2\langle g_1-g_0,R\rangle.
    \end{align*}
The first term $\langle g_1-g_0,g_y-g_{1-y}\rangle$ is determined by the true label $y$ and does not depend on the noise $R$.
For the second term, notice that $\langle g_1-g_0,R\rangle$ only depends on the projection of $R$ in the direction of $g_1-g_0$.
In other words, components of $R$ on other orthogonal directions do not affect the decision of shortest distance attack. 

The above calculation suggests that $g_0-g_1$ is a special direction,
which inspires us to consider noise perturbation mechanism that focuses on the direction of $g_0-g_1$.

\subsection{Generic label protection mechanism}\label{sec:grad_perturb_proof}

Based on the insights inspired by the white-box attack,
we propose \textsf{GradPerturb}, a general gradient-perturbation scheme for \lp{} in Algorithm~\ref{alg:noisy_gradient}.
\textsf{GradPerturb} is specified by a distribution $U$ over $\mathbb{R}$ which controls the noise scale.
With gradients $g_0$ and $g_1$,
\textsf{GradPerturb} adds noise to $g_y$ in the direction of $g_0-g_1$.

\begin{algorithm}[!t]
\caption{\textsf{GradPerturb} equipped with distribution $U$}\label{alg:noisy_gradient}
\KwIn{True label $y\in \{0,1\}$. Gradients $g_0$ and $g_1$ with respect to label $0$ and $1$.}
\lp{} samples $u\sim U$.\\
\Return{$g_y+u\cdot (g_{1-y}-g_y)$.}
\end{algorithm}

We explore two choices of the distribution $U$:
\begin{itemize}
    \item (Laplace perturbation) $U$ is $\mathsf{Lap}(b)$, i.e., the Laplace distribution centered at $0$ with scale $b$.
    The PDF of $\mathsf{Lap}(b)$ is given by $p(x|b)=\frac{1}{2b}e^{-\frac{|x|}{b}}$.
    \item (Discrete perturbation) $U$ is a Bernoulli Distribution $\mathsf{Bern}(p)$ for $p\in[0,\frac{1}{2}]$. Namely, $U$ take values from $\{0,1\}$ with $\Pr[U=1]=p$ and $\Pr[U=0]=1-p$.
\end{itemize}
We will show that both mechanisms are differential private. 
The detailed proof is deferred to Appendix~\ref{sec:missing_proofs}.
\begin{lemma}[Laplace perturbation is differentially private]\label{lem:laplace_perturb}
Let $U=\mathsf{Lap}(b)$. Fix $\epsilon>0$. When $b\geq 1/\epsilon$, $\textsf{GradPerturb}$ is $(\epsilon,0)$-DP with respect to the label $y$.
\end{lemma}


\begin{lemma}[Discrete perturbation is differentially private]\label{lem:discrete_perturb}
Let $U=\mathsf{Bern}(p)$. Then $\textsf{GradPerturb}$ is $(\ln \frac{1-p}{p},0)$-DP with respect to the label $y$.
\end{lemma}

\paragraph{Remark.} It turns out that the discrete perturbation coincides with the randomized response mechanism used in~\cite{ggk+21}. Indeed, our definition of transcript DP is closely related to the concept of label DP~\citep{ggk+21}, because we require neighbouring datasets to differ at only one label. Consequently, transcript DP is a stronger notion of privacy in the sense that given any MPC protocol that is $(\epsilon,\delta)$-transcript DP, we can change it into a centralized $(\epsilon,\delta)$-label DP algorithm.
It is an interesting question that if the converse is true, namely does there exist a generic reduction that can turn a label DP algorithm into a distributed algorithm with transcript DP guarantee? 

\section{Differentially private split learning}
\label{sec:algorithm}

We now use \textsf{GradPerturb} as a building block to design a transcript differentially private split learning algorithm.

\subsection{Transcript Private Split Learning}
Our method \textsf{TPSL} (Transcript Private Split Learning) is presented in Algorithm~\ref{alg:dp_sgd}.
Compared to the original split learning algorithm (Algorithm~\ref{alg:sgd}),
\textsf{TPSL} perturbs not only the gradients $g_{b,i}$,
but also the model updates of \lp{}.
Note that we need to also perturb the model updates so that in our proof  we can avoid DP composition theorems~\citep{drgv10} that would lead to a DP budget that depends on the number of batches or the number of samples.
\setlength{\algomargin}{15pt}
\begin{algorithm*}[!ht]
 \caption{Transcript Private Split Learning}\label{alg:dp_sgd}
\KwIn{Features $X$ held by \np. Labels $Y$ held by \lp. $|X|=|Y|=n$.}
\KwPar{Number of batches $B$. Step size $\eta$. Loss function $\ell:[0,1]\times \{0,1\}\rightarrow \mathbb{R}_{\geq 0}$. Distribution $U$ over $\mathbb{R}$.}
Training samples are divided into $B$ batches.\\
$\theta_N,\theta_L$ are initialized with Gaussian initialization.\\
\For{$b=1,2,\cdots,B$}{
\tcp{Forward phase}
\np{} computes embeddings $e_{b,i}:=f_{\theta_N}(\mathbf{x}_{b,i})$ for $i=1,\cdots,n/B$.\\
\np{} sends $\{e_{b,1},\cdots,e_{b,n/B}\}$ to \lp.\\
\tcp{Backward phase}
\For{$i=1,\cdots,n/B$}{
\lp{} computes the prediction of the neural network $h$ as $\tilde y_{b,i} := h_{\theta_L}(f_{\theta_N}(\mathbf{x}_{b,N}))$.\\
\For{$j\in \{0,1\}$}{
\lp{} computes losses $\ell_{b,i,j} := \ell(\tilde y_{b,i}, j)$.\\
\lp{} computes derivatives $g_{b,i,j} := \frac{\partial \ell_{b,i,j}}{\partial e_{b,i}}$ and $u_{b,i,j}:=\frac{\partial \ell_{b,i,j}}{\partial \theta_L}$.
}
\lp{} computes $\tilde g_{b,i}:=\textsf{GradPerturb}(y_{b,i},g_{b,i,0},g_{b,i,1})$
and $\tilde u_{b,i}:=\textsf{GradPerturb}(y_{b,i},u_{b,i,0},u_{b,i,1})$ using Algorithm~\ref{alg:noisy_gradient}.
}
\lp{} sends $\{
\tilde g_{b,1},\cdots,\tilde g_{b,n/B}\}$ to \np.\\
\tcp{Model updates}
\lp{} updates $\theta_L$ by $\theta_L\leftarrow \theta_L- \eta\cdot \frac{B}{n}\cdot \sum_{i=1}^{n/B}\tilde u_{b,i}$.\\
\np{} computes $\frac{\partial \ell_{b,i}}{\partial \theta_N} = \tilde g_i^\top \frac{\partial e_{b,i}}{\theta_N}$ and then  updates $\theta_N$ by $\theta_N\leftarrow \theta_N- \eta\cdot \frac{B}{n}\cdot \sum_{i=1}^{n/B}\frac{\partial \ell_{b,i}}{\partial \theta_N}$.\\
}
\end{algorithm*}

The privacy of \textsf{TPSL} is given by the following lemma:
\begin{lemma}\label{lem:dp_main}
Fix $\epsilon>0$. If \textsf{GradPerturb} is $(\epsilon,0)$-DP under the distribution $U$,
then \textsf{TPSL} is $(2\epsilon,0)$-transcript DP.
\end{lemma}

\begin{proof}
The transcript of \textsf{TPSL} is $S_t:=\{e_{b,i}\}_{b\in[B],i\in[n/B]}\cup \{\tilde g_{b,i}\}_{b\in[B],i\in[n/B]}$.
We further consider the model updates in the $B$ batches,
which we denote as $S_m:=\{\Delta \theta_N^{(b)},\Delta \theta_L^{(b)}\}_{b=1}^B$.
Then $S_t,S_m$ are random functions of input dataset $D$.
We denote the output of $S_t$ and $S_m$ with input $D$ by $(S_t,S_m)(D)$.

Next we prove that for any neighbouring $D,D'$ and any assignment $s$, we have
\begin{align}\label{eq:dp_goal}
    \Pr[(S_t,S_m)(D)=s]\leq e^{2\epsilon} \cdot \Pr[(S_t,S_m)(D')=s].
\end{align}
Let $(b^*,i^*)$ be the index of the different label in $D$ and $D'$,
i.e., $y_{b^*,i^*}=1-y_{b^*,i^*}'$.
Then in batch $1,2,\cdots,b^*-1$,
$D$ and $D'$ are identical, hence the probability restricted on the first $b^*-1$ batches is the same.
Formally,
let $S_{-}$ be the parts of $S_t,S_m$ that are in the first $b^*-1$ batches, 
then we have
{\scriptsize
\begin{align*}
    \Pr\left[[(S_t,S_m)(D)]_{S_{-}}=[s]_{S_{-}}\right]=\Pr\left[[(S_t,S_m)(D')]_{S_{-}}=[s]_{S_{-}}\right].
\end{align*}
}


Furthermore, if $S_t$ and $S_m$ are the same for the first $b^*$ batches,
then the model weights after the $b^*$-th batch will be the same for $D$ and $D'$,
because all previous model updates are the same.
Since the rest training data is identical in $D$ and $D'$, 
the remaining part in $(S_t,S_m)$ will also be identical.
Formally,
define $S_*$, $S_{+}$ as
{\small
\begin{align*}
    S_{*}= & ~\{e_{b^*,i}\}_{i\in [1, n/B]}\cup  \{\tilde g_{b^*,i}\}_{i\in [1, n/B]}\cup\{\Delta\theta_N^{(b^*)},\Delta\theta_L^{(b^*)}\},\\
    S_{+}= & ~\{e_{b,i}\}_{b \in[b^*+1,B],i\in [1, n/B]}\cup \{\tilde g_{b,i}\}_{b \in [b^*+1,B],i\in [1, n/B]}\\
     &\cup\{\Delta\theta_N^{(b)},\Delta\theta_L^{(b)}\}_{b \in [b^*+1,B]}.\\
\end{align*}
}
Then we have
{\scriptsize
\begin{align*}
    &~\Pr\left[[(S_t,S_m)(D)]_{S_{+}}=[s]_{S_{+}}\Big\vert [(S_t,S_m)(D)]_{S_{-}\cup S_{*}}=[s]_{S_{-}\cup S_{*}} \right]\\
    = & ~\Pr\left[[(S_t,S_m)(D')]_{S_{+}}=[s]_{S_{+}}\Big\vert [(S_t,S_m)(D')]_{S_{-}\cup S_{*}}=[s]_{S_{-}\cup S_{*}} \right].
\end{align*}
}
So we conclude that
{\small
\begin{align*}
    &~\frac{\Pr[(S_t,S_m)(D)=s]}{\Pr[(S_t,S_m)(D')=s]}\\
    = & ~\frac{\Pr\left[[(S_t,S_m)(D)]_{S_{*}}=[s]_{S_{*}}\Big\vert [(S_t,S_m)(D)]_{S_{-}}=[s]_{S_{-}} \right]}{\Pr\left[[(S_t,S_m)(D')]_{S_{*}}=[s]_{S_{*}}\Big\vert [(S_t,S_m)(D')]_{S_{-}}=[s]_{S_{-}} \right]}
\end{align*}
}
The RHS is concerning the ratio of the probability that the transcript and model updates in the $b^*$-th batch are the same.
Let $g$ be the value of $\tilde g_{b^*,i^*}$ in $s$,
and $u$ be the value of $\tilde u_{b^*,i^*}$ in $s$.
Let $\mathcal{E}_{D,g,u}$ be the event of $\textsf{GradPerturb}(y_{b^*,i^*},g_0,g_1)=g,\textsf{GradPerturb}(y_{b^*,i^*},u_0,u_1)=u$,
and $\mathcal{E}_{D',g,u}$ be the event of 
$\textsf{GradPerturb}(1-y_{b^*,i^*},g_0,g_1)=g,\textsf{GradPerturb}(1-y_{b^*,i^*},u_0,u_1)=u$.
Since all the training samples except for the $i^*$-th one is the same in $D$ and $D'$,
we have
{\small
\begin{align*}
    &~\frac{\Pr\left[[(S_t,S_m)(D)]_{S_{*}}=[s]_{S_{*}}\Big\vert [(S_t,S_m)(D)]_{S_{-}}=[s]_{S_{-}} \right]}{\Pr\left[[(S_t,S_m)(D')]_{S_{*}}=[s]_{S_{*}}\Big\vert [(S_t,S_m)(D')]_{S_{-}}=[s]_{S_{-}} \right]} \\
    = & ~\frac{\Pr[\mathcal{E}_{D,g,u}]}{\Pr[\mathcal{E}_{D',g,u}]}.
\end{align*}
}
Since $\textsf{GradPerturb}$ is $(\epsilon,0)$-DP with the distribution $U$,
we have
\begin{align*}
\frac{\Pr[\mathcal{E}_{D,g,u}]}{\Pr[\mathcal{E}_{D',g,u}]}
  \leq e^{2\epsilon},
\end{align*}
which completes the proof of  Eq.~\eqref{eq:dp_goal}.

Then we show $(S_t,S_m)$ is $(2\epsilon,0)$-DP.
For any subset $S$ of possible assignments of $(S_t,S_m)$,
we have
{\small
\begin{align*}
    \Pr[(S_t,S_m)(D)\in S] = &~ \int_{s\in S} \Pr[(S_t,S_m)(D)=s] ds\\
    \leq &~ \int_{s\in S} e^{2\epsilon}\Pr[(S_t,S_m)(D')=s] ds\\
    = & ~e^{2\epsilon}\Pr[(S_t,S_m)(D')\in S],
\end{align*}
}
where the second step uses Eq.~\eqref{eq:dp_goal}.
Hence $(S_t,S_m)$ is $(2\epsilon,0)$-DP.
As a consequence, \textsf{TPSL} is $(2\epsilon,0)$-transcript DP,
which completes the proof of the lemma.
\end{proof}


Note that the privacy budget of \textsf{TPSL} only depends on the budget of \textsf{GradPerturb},
but not on the number of samples $n$ or number of batches $B$. 

\subsection{Improve privacy budget analysis}\label{sec:dp_improve}
From the above proof, we see that if we can couple the generation of $\tilde g_{b,i}$ and $\tilde u_{b,i}$,
then it is possible to further tighten the bound of privacy budget.
In the setting of neural networks, we can achieve this by adding noise on the final hidden layer.
Formally, let $\theta^*$ be the parameters of the last hidden layers of the sub-network $h_{\theta_L}(\cdot)$ held by \lp{},
then we can rewrite $h_{\theta_L}(\cdot)$ as $h_1(\theta^*,h_2(\bar \theta_L,\cdot))$ where $\bar\theta_L$ is the remaining parameters in $\theta_L$.
From the chain rule of derivatives,
we have
\begin{align*}
    g_{b,i} =& \frac{\partial \ell_{b,i}}{\partial e_{b,i}} = \frac{\partial \ell(h_1(\theta_*,h_2(\bar\theta_L,e_{b,i})),y)}{\partial e_{b,i}} = \frac{\partial \ell_{b,i}}{\partial h_1}\frac{\partial h_1}{\partial e_{b,i}},\\
        u_{b,i} = & \frac{\partial \ell_{b,i}}{\partial \theta_L} = \frac{\partial \ell(h_1(\theta_*,h_2(\bar\theta_L,e_{b,i})),y)}{\partial \theta_L} = \frac{\partial \ell_{b,i}}{\partial h_1}\frac{\partial h_1}{\partial \theta_L}.
\end{align*}
Hence we can apply \textsf{GradPerturb} once on $\frac{\partial \ell_{b,i}}{\partial h_1}$, i.e.,
\lp{} computes $v_{b,i,j}:=\frac{\partial \ell_{b,i,j}}{\partial h_1}$ for $j\in \{0,1\}$ and 
$\tilde v_{b,i}=\textsf{GradPerturb}(y_{b,i},v_{b,i,0},v_{b,i,1})$.
Then \lp{} computes $\tilde g_{b,i}=\tilde v_{b,i}\frac{\partial h_1}{\partial e_{b,i}}$ and $\tilde u_{b,i}=\tilde v_{b,i}\frac{\partial h_1}{\partial \theta_L}$.
With similar proof,
now we can show that when \textsf{GradPerturb} is $(\epsilon,0)$-DP,
we have
{\scriptsize
\begin{align*}
    \frac{\Pr\left[[(S_t,S_m)(D)]_{S_{*}}=[s]_{S_{*}}\Big\vert [(S_t,S_m)(D)]_{S_{-}}=[s]_{S_{-}} \right]}{\Pr\left[[(S_t,S_m)(D')]_{S_{*}}=[s]_{S_{*}}\Big\vert [(S_t,S_m)(D')]_{S_{-}}=[s]_{S_{-}} \right]}\leq e^{\epsilon}.
\end{align*}
}
To summarize, by slightly modifying Algorithm~\ref{alg:dp_sgd},
we can improve the bound on privacy budget to $(\epsilon,0)$ from $(2\epsilon,0)$. 
\section{Experiments}
In this section, we report the experiment results of split learning algorithms. We train a modified Wide\&Deep model \citep{ckh+16} on two real-world large-scale datasets, Avazu~\citep{avazu} and Criteo~\citep{criteo}.
The model $f$ held by the non-label party \np{} is the embedding layer for input features plus two 128-unit ReLU activated multi-layer perceptron (MLP) layers. 
The model $h$ held by the label party \lp{} is two layers of 128-unit ReLU activated MLP.
We use SGD with batch size $8,192$ for training.

\subsection{Label leakage in split learning}\label{sec:label_attack_split_learning}
We first show that the naive implementation of split learning, for example Algorithm~\ref{alg:sgd},
does leak label information, which verify the findings of \cite{lsy+21}.
We consider three kinds of attack: norm attack (\texttt{NA}), spectral attack (\texttt{SA}) and shortest distance attack (\texttt{SDA}). Norm attack ~\citep{lsy+21} is a simple heuristic for black-box attack,
and has been shown to be effective on imbalanced datasets like Avazu and Criteo.
Spectral attack~\citep{tlm18,syy+21} is a 2-clustering algorithm that works as a black-box attack.
More details about these attacks can be found in Appendix~\ref{sec:black_box_attack}.
Shortest distance attack is the white-box attack introduced in Section~\ref{sec:white_box_attack}.

We use \emph{attack AUC} for the evaluation of these attacks. Detailed description of attack AUC is in Appendix~\ref{sec:attack_auc}. 
Attack AUC close to 1 means the attack can almost perfectly reconstruct the labels,
while close to 0.5 attack AUC means the attack behaves similarly to random guess.

We evaluate these attacks on Algorithm~\ref{alg:sgd} with the Avazu dataset and report the results in Table~\ref{tab:label_leakage}.
For reference we also report the test AUC.
We can clearly see that all three attacks can recover the true label almost perfectly.
Furthermore, white-box attack (i.e. shortest distance attack) achieves better performance compared to black-box attacks.

\begin{table}[ht!]
\centering
\begin{tabular}{|l|l|l|l|} 
\hline
           test AUC & \texttt{NA} AUC & \texttt{SA} AUC & \texttt{SDA} AUC  \\ 
\hline
  0.7523   & 0.9897        & 0.9890            & 1.0000                      \\
\hline
\end{tabular}
\caption{Attacks on non-private split learning on Avazu.}\label{tab:label_leakage}
\vspace{-6.05mm}
\end{table}


\subsection{Results of transcript DP split learning}\label{sec:exp_dp_split}
We evaluate the performance of the two approaches we propose: Laplace perturbation (\texttt{Laplace}) and discrete perturbation (\texttt{Discrete}).
For different choices of $\epsilon$,
we report test performance as the utility metric.
We also report the attack AUC for the attacks introduced in Section~\ref{sec:label_attack_split_learning}.
The results on Avazu are presented in Table~\ref{tab:dp_algo_results},
and the results on Criteo are are presented in Table~\ref{tab:experimental_results_criteo}.
For comparison,
we also include a baseline of isotropic Gaussian perturbation (\texttt{Gaussian}). 
That is,
we use $\textsf{GradPerturb}(y,g_0,g_1)=g_y+r$ where $r$ is an isotropic Gaussian vector.
Notice that Gaussian mechanism cannot achieve pure DP (i.e. $\delta=0$),
hence our framework of Lemma~\ref{lem:dp_main} does not apply and we do not report the value of $\epsilon$ for Gaussian perturbation.
We apply $\textsf{GradPerturb}$ on the final hidden layer as described in Section~\ref{sec:dp_improve}.


We can see that shortest distance attack is more informative than the other two black-box attacks.
For instance,
when $\epsilon = 1$,
the attack AUC of norm attack and spectral attack of Laplace perturbation is about $0.57$,
which seems to be low.
However, this is a false sense of security: the shortest distance attack has attack AUC around $0.7$.

{\tiny
\begin{table*}[t]
\centering
\resizebox{\columnwidth}{!}{
\begin{tabular}{|l|l|l|l|l|l|l|} 
\hline
\texttt{SDA} AUC level           & Algorithms & \texttt{SDA} AUC & Test AUC & $\epsilon$ & \texttt{NA} AUC & \texttt{SA} AUC  \\ 
\hline
\multirow{3}{*}{Low}    & \texttt{Laplace}   & 0.5237   & 0.6689   & 0.1        & 0.4932   & 0.4847    \\ 
\cline{2-7}
                        & \texttt{Discrete}   & 0.5274   & 0.6293   & 0.1        & 0.4997   & 0.4997    \\ 
\cline{2-7}
                        & \texttt{Gaussian}   & 0.5252 & 0.6657   & --         & 0.5117  & 0.4999  \\ 
\hline
\multirow{3}{*}{Medium} & \texttt{Laplace}    & 0.6961 & 0.7416   & 1          & 0.5784 & 0.5743  \\ 
\cline{2-7}
                        & \texttt{Discrete}   & 0.7789  & 0.7342   & 1          & 0.4806 & 0.4796  \\ 
\cline{2-7}
                        & \texttt{Gaussian}   & 0.6455 & 0.7384   & --         & 0.5440 & 0.5363 \\ 
\hline
\multirow{3}{*}{High}   & \texttt{Laplace}    & 0.9622  & 0.7518   & 10         & 0.9640 & 0.9966   \\ 
\cline{2-7}
                        & \texttt{Discrete}   & 0.9999 & 0.7523   & 10         & 0.9898 & 0.9893  \\ 
\cline{2-7}
                        & \texttt{Gaussian}   & 1        & 0.7522   & --         & 0.9902  & 0.9899  \\
\hline
\end{tabular}
}
\caption{Performance of different transcript DP split learning algorithms on Avazu.}\label{tab:dp_algo_results}
\end{table*}
}

\begin{table*}[t]
\centering
\resizebox{\columnwidth}{!}{
\begin{tabular}{|l|l|l|l|l|l|l|l}\hline
\texttt{SDA} AUC Level & Algorithms & \texttt{SDA} AUC &Test AUC &$\epsilon$ &\texttt{NA} AUC &\texttt{SA} AUC \\\hline
\multirow{3}{*}{Low} &\texttt{Laplace} &0.5236 &{0.7095} &0.1 &0.5011 &0.5011 \\\cline{2-7}
&\texttt{Discrete} &0.5258 &0.6037 &0.1 &0.5138 &0.5215 \\\cline{2-7}
&\texttt{Gaussian} &0.5264 &0.6614 &- &0.4985 &0.5017 \\ \hline
\multirow{3}{*}{Medium} &\texttt{Laplace} &0.6964 &{0.7639} &1 &0.5525 &0.5525 \\\cline{2-7}
&\texttt{Discrete} &0.7310 &0.7497 &1 &0.7317 &0.7284 \\\cline{2-7}
&\texttt{Gaussian} &0.6409 &0.7550 &- &0.5017 &0.5031 \\ \hline
\multirow{3}{*}{High} &\texttt{Laplace} &0.9965 &0.7784 &10 &0.8761 &0.8766 \\\cline{2-7}
&\texttt{Discrete} &0.9999 &{0.7792} &10 &0.9202 &0.9204 \\\cline{2-7}
&\texttt{Gaussian} &1.0 &0.7786 &- &0.9038 &0.9159 \\
\hline
\end{tabular}
}
\caption{Performance of different transcript-DP split learning algorithms on
Criteo.}\label{tab:experimental_results_criteo}
\end{table*}

To better illustrate the utility-privacy trade-off,
we plot the curve of utility (measured by test AUC) vs. privacy (measured by privacy budget $\epsilon$) in Figure~\ref{fig:avazu_tradeoff} (a)(c) for \texttt{Laplace} and \texttt{Discrete}. 
To compare with \texttt{Gaussian},
we use \texttt{SDA} as the privacy metric and draw Figure~\ref{fig:avazu_tradeoff} (b)(d).

We observe that \texttt{Laplace} and \texttt{Gaussian} have similar
utility-privacy trade-off curves,
which confirms the insight that to protect label privacy, it is sufficient to add noise in the direction of $g_0-g_1$.
Furthermore,
\texttt{Laplace} achieves good utility within reasonable range of $\epsilon$.
Compared to the non-private baseline,
when $\epsilon = 1$,
the test AUC only drops by $1.4\%$.

We also observe that \texttt{Discrete} has worse utility-privacy trade-off than \texttt{Laplace} and \texttt{Gaussian}.
One explanation is that \texttt{Discrete} generates biased estimators, namely $\E[\textsf{GradPerturb}(y,g_0,g_1)]\neq g_y$ for \texttt{Discrete}. However,
the other two methods provide unbiased gradient estimator.

\begin{figure*}[ht!]
      \begin{minipage}[t]{0.5\linewidth}
  \centering
    \includegraphics[width=0.8\linewidth]{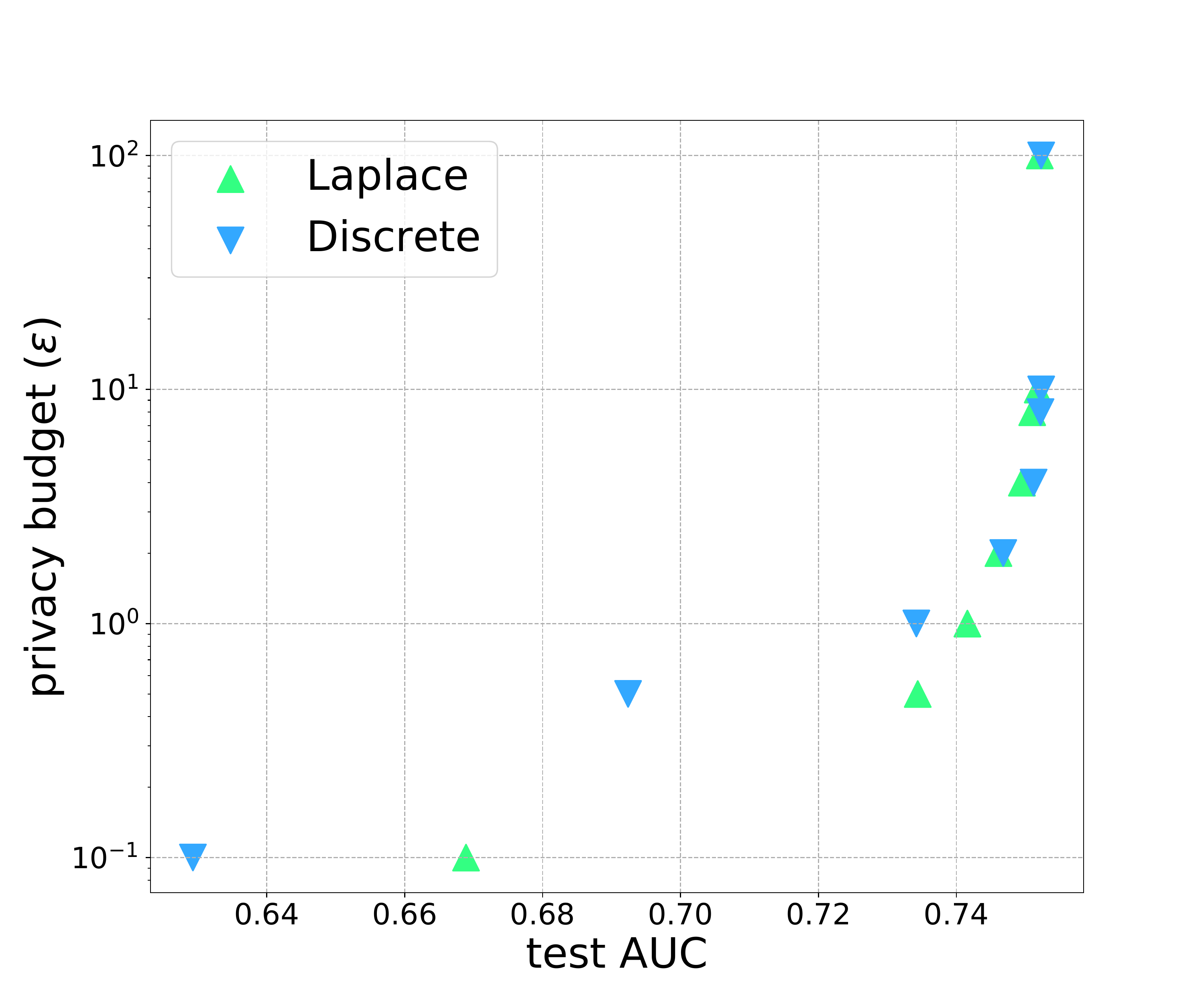}
      \caption*{(a): Trade-off of transcript DP algorithms\\ on Avazu.}
  \end{minipage}
    \begin{minipage}[t]{0.5\linewidth}
  \centering
    \includegraphics[width=0.8\linewidth]{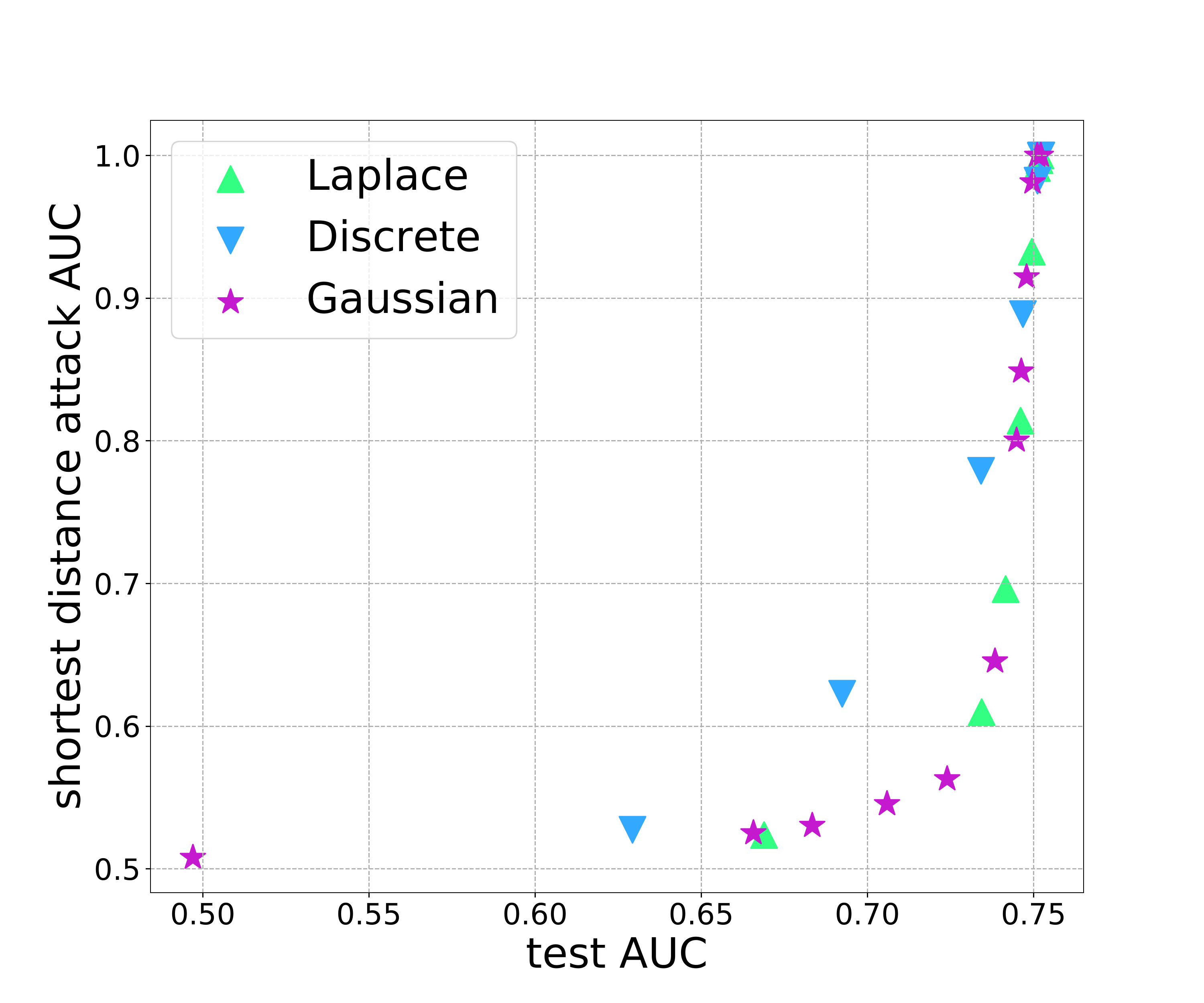}
      \caption*{(b):Trade-off measured by shortest attack AUC on Avazu.}
  \end{minipage}
      \begin{minipage}[t]{0.5\linewidth}
  \centering
    \includegraphics[width=0.8\linewidth]{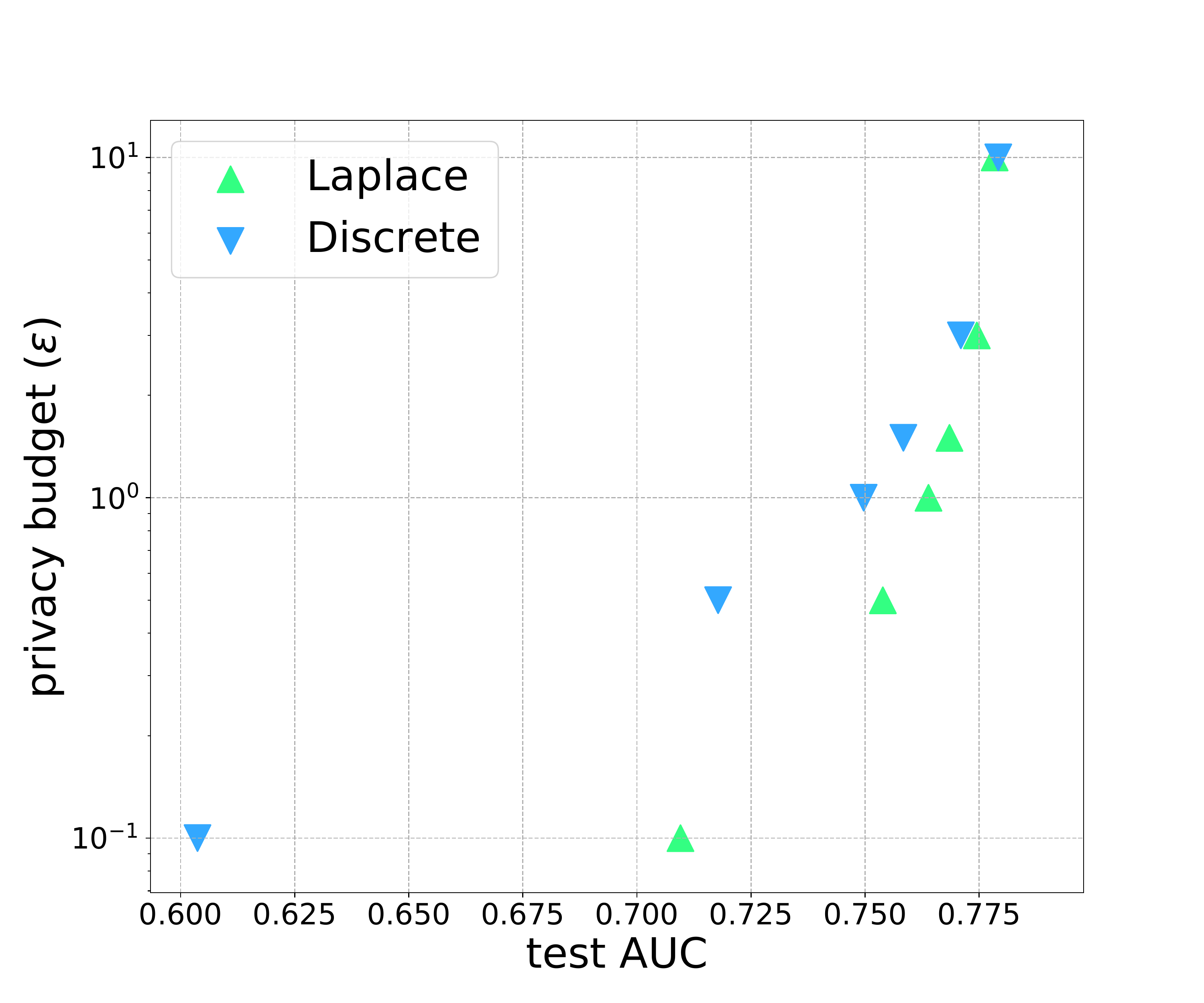}
      \caption*{(c): Trade-off of transcript DP algorithms\\ on Criteo.}
  \end{minipage}
    \begin{minipage}[t]{0.5\linewidth}
  \centering
    \includegraphics[width=0.8\linewidth]{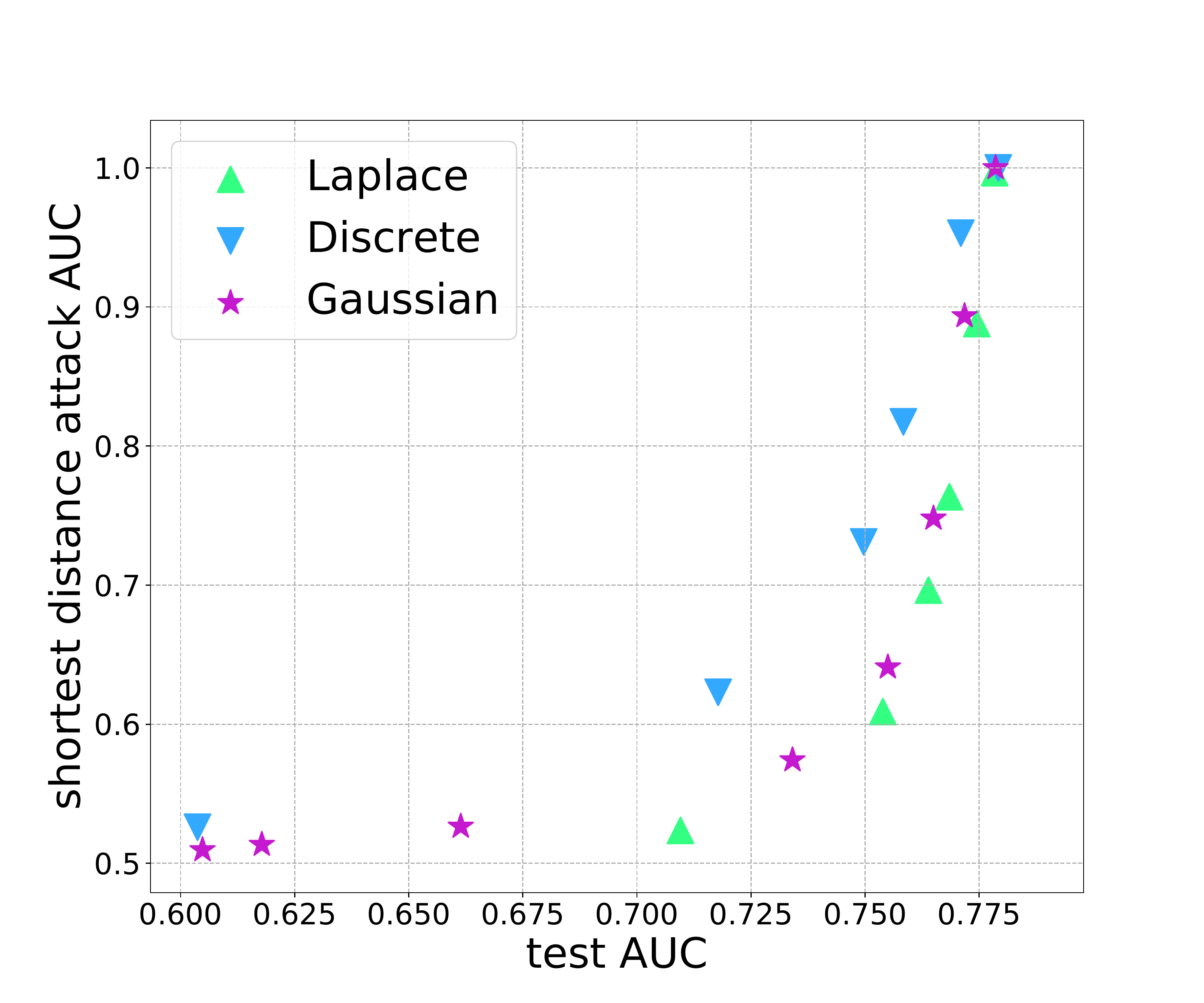}
      \caption*{(d):UTrade-off measured by shortest attack AUC on Criteo.}
  \end{minipage}
      \caption{Utility-privacy trade-off.}\label{fig:avazu_tradeoff}
  \end{figure*}



\section{Extensions}
\subsection{Go beyond online learning}

Our learning algorithm \textsf{TPSL} is an online training algorithm~\citep{dhs11} in the sense that data is accessed in a sequential order.
Online learning is widely used in practice to train neural networks due to its scalability and efficiency.
Our proof of differential privacy makes use of the property of online learning that each training sample is used only once in the entire procedure. 

However, our perturbation-based methods can also be used even if samples can be used multiple times, such as in multi-epoch training.
The strategy of the label party \lp{} needs to be modified slightly:
when \lp{} calls \textsf{GradPerturb} for one training sample for the first time, \lp{} needs to record the random scalar generated by \textsf{GradPerturb}.
Whenever the same sample is used again,
instead of running a fresh \textsf{GradPerturb},
\lp{} just uses the recorded random scalar to generate the perturbed gradient.
It can be shown that this algorithm is still transcript differentially private.
\subsection{Multi-label classification}
Our learning scheme can be easily extended to multi-label classification.
Notice that the proof of Lemma~\ref{lem:dp_main} does not depend on domain of labels $\mathcal{Y}$ being $\{0,1\}$.
When $\mathcal{Y}=\{0,1,2,\cdots,k-1\} $ with $k>2$,
as long as we can design differentially private perturbation scheme \textsf{GradPerturb},
then \textsf{TPSL} is still differentially private by Lemma~\ref{lem:dp_main}.

To deal with multiple labels, we design Algorithm~\ref{alg:noisy_gradient_multi}. It is a natural generalization of Algorithm~\ref{alg:noisy_gradient}.
Instead of only adding noise in the direction of $g_{1-y}-g_{y}$ as in the binary case,
we now add noise on each $g_{y'}-g_y$ for $y'=0,1,\cdots,k-1$.
Notice that the output of Algorithm~\ref{alg:noisy_gradient} can be interpreted as
{\small
\begin{align*}
    g_y+\sum_{i=0}^{k-1}u_i\cdot g_{i}
    = (1+u_y-\sum_{y'\neq y} u_{y'})g_y + \sum_{y'\neq y}u_{y'}(g_{y'}-g_y).
\end{align*}
}
The first term $(1+u_y-\sum_{y'\neq y} u_{y'})g_y$ is a scaled version of the correct gradient $g_y$.
The second term $\sum_{y'\neq y}u_{y'}(g_{y'}-g_y)$ is just the noise added on the direction of each $g_{y'}-g_{y}$.
\begin{algorithm}[!t]
 \caption{\textsf{ GradPerturb} with multiple labels equipped with distribution $U$ over $\mathbb{R}^{k}$}\label{alg:noisy_gradient_multi}
\KwIn{True label $y\in\{0,1,\cdots,k-1\}$. Gradients $g_0,g_1,\cdots,g_{k-1}$ corresponds to label $0,1,\cdots,k-1$ respectively.}
\lp{} samples $u\sim U$.\\
\Return{$g_y+\sum_{i=0}^{k-1}u_i\cdot g_{i}$.}
\end{algorithm}

We give two instances of the distribution $U$ that makes Algorithm~\ref{alg:noisy_gradient_multi} $(\epsilon,0)$-DP. As the consequence, \textsf{TPSL} is also $(2\epsilon,0)$-DP when equipped with such $U$.
\begin{itemize}
    \item (Multi-class Laplace perturbation) $U=(U_0,U_1\cdots,U_{k-1})$ where $U_i$ are i.i.d sampled from $\mathsf{Lap}(2/\epsilon)$.
    \item(Multi-class discrete perturbation) Let $e_i\in \mathbb{R}^d$ be the $i$-th (starting from 0) unit vector, that is, the $i$-th coordinate is 1 and all the rest coordinates are 0.
    Then $U$ is chosen from $\{e_i-e_y\}_{i=0}^{k-1}$ with probability
    \begin{align*}
        \Pr[U=e_i-e_y]=\begin{cases} \frac{e^{\epsilon}}{e^{\epsilon}+k-1} & \mbox{if $i=y$},\\
        \frac{1}{e^{\epsilon}+k-1} & \mbox{else.}
        \end{cases}
    \end{align*}
\end{itemize}

The DP guarantees are given in the following lemmas. 
The detailed proof is deferred to Appendix~\ref{sec:missing_proofs}.
\begin{lemma}\label{lem:laplace_perturb_multi}
\textsf{GradPerturb} equipped with multi-class Laplace perturbation is $(\epsilon,0)$-DP.
\end{lemma}

\begin{lemma}\label{lem:discrete_perturb_multi}
\textsf{GradPerturb} equipped with multi-class discrete perturbation is $(\epsilon,0)$-DP.
\end{lemma}

\section{Related Works}
\textbf{Differentially Private Machine Learning.}
Differential privacy and machine learning are closely related~\citep{dfh+15},
and there has been a long line of research on designing differentially private machine learning models~\citep{cms11,l11,vsbh13,sazl18,pth+20}.
For the case of deep learning,
since the breakthrough work of \citet{acg+16} of DP-SGD,
there has been many research works on differentially private deep learning~\citep{pae+16,psm+18,mrtz18,tb20}.
Recently, \citet{ggk+21} and \citet{mmp+21} studied the case where only labels are considered sensitive, and showed that with the same privacy budget, it can achieve better model performance compared to DP-SGD.
\vspace{-0.7mm}
In these works, differential privacy is usually enforced on the weights of machine learning models.
This is different from our setting of Vertical Federated Learning, where the adversary can also access the shared messages.
Hence we also need to make the communication differentially private.
\vspace{-0.8mm}

\textbf{Differentially Private Vertical Federated Learning.}
Prior to our work,
researchers have explored  using differential privacy in Vertical Federated Learning.
For logistic regression, \citet{wlh+20} proposed using Gaussian mechanism on the intermediate results to improve running time while maintain privacy.
For neural networks,
\citet{cjs+20} added differentially private noise on each layer and showed that with sufficiently large noise, the privacy budget can be met.
Compared with these works,
our method can work with neural networks,
and since we focus on protecting label privacy, 
no extra gradient clipping is needed, which is in contrast to \citet{cjs+20} or DP-SGD~\citep{acg+16}.

\section{Conclusion}
In this work, we consider split learning in the multi-party computation setting and introduce the notion of transcript differential privacy to measure the label privacy. 
Inspired by the white-box attack, 
we propose a generic gradient-perturbation scheme \textsf{GradPerturb} that adds noise only on the optimal direction.
Based on \textsf{GradPerturb},
we propose our split learning framework of \textsf{TPSL}.
To the best of our knowledge,
this is the first split learning algorithm that provides provable differential privacy guarantees. 
Experiments over large-scale datasets show that with Laplace perturbation, \textsf{TPSL} can achieve good utility-privacy trade-off.

It remains an open question that how we can build a split learning protocol so that the input features can also be protected. 
A straightforward implementation to enforce differential privacy on input features is to add noise not only on the gradients in the backward phase, but also on the embeddings in the forward phase. 
It is would be interesting if we can find smarter ways than adding isotropic noise.
Another interesting question is that if we can build up connections between the privacy notions we introduce, including transcript DP and multiple attack AUCs.

\bibliographystyle{plainnat}
\bibliography{ref}

\newpage

\appendix
\section{Label inference attacks from gradients}\label{sec:black_box_attack}
In this section,
we introduce two black-box attacks, norm attack and spectral attack.
We then introduce attack AUC as the evaluation metric for these attacks.

\subsection{Norm attack}
Norm attack~\citep{lsy+21} is a simple heuristic attack that is found to be effective on imbalanced datasets.
On gradient $g$,
the output of norm attack is a scalar 
\begin{align*}
    \texttt{NA}(g):=\|g\|_2^2.
\end{align*}
It turns out that positive samples tend to have larger norm compared to negative samples.
Hence norm attack can be used to recover the label information.

\subsection{Spectral attack}

Spectral attack is a singular value decomposition (SVD) based outlier detection method introduced by Tran, Li and Madry~\citep{tlm18}.
In particular,
they show that
\begin{lemma}[Lemma 3.1, Definition 3.1 in \cite{tlm18}]\label{lem:spectral_attack}
Fix $0<\epsilon<\frac{1}{2}$.
Let $D$, $W$ be two distributions over $\mathbb{R}^d$ with mean $\mu_D,\mu_W$ and covariance matrices $\Sigma_D,\Sigma_W$.
Let $F$ be a mixture distribution given by $F=(1-\epsilon)D+\epsilon W$.
If $\|\mu_D-\mu_W\|_2^2\geq \frac{6\sigma^2}{\epsilon}$,
then the following statement holds:
let $\mu_F$ be the mean of $F$ and $v$ be the top singular vector of the covariance matrix of $F$,
then there exists $t$>0 so that 
\begin{align*}
    \Pr_{X\sim D}[|\langle X-\mu_F,v\rangle |>t] < & ~\epsilon,\\
    \Pr_{X\sim W}[|\langle X-\mu_F,v\rangle |<t]< & ~\epsilon.
\end{align*}
\end{lemma}

\citet{syy+21} shows that the above lemma can be used as a 2-clustering algorithm: let $D$ and $W$ be two distributions that we would like to distinguish, and we are given a collection of samples $F$ chosen from $D$ and $W$. 
Then we can use the value $|\langle X-\mu_F,v\rangle |$ to cluster the samples.

We use this idea to build a black-box attack, which we call spectral attack.
In particular,
we think of $F$ as the distribution of gradients of positive samples,
and $W$ as the distribution of gradients of negative samples.
In our experiment,
we conduct spectral attack in every mini-batch.
We estimate $\mu_F$ and $v$ by computing the empirical mean and covariance matrix with data in the mini-batch.
Then the output of spectral attack is computed by
\begin{align*}
    \texttt{SA}(g):=|\langle g-\mu_F,v\rangle |.
\end{align*}

\subsection{Attack AUC}\label{sec:attack_auc}
As the output of these attacks may not be Boolean,
standard metrics like accuracy, precision and recall do not evaluate the effectiveness of the attacks well.
In order to get rid of the problem of setting threshold for binary decision,
we instead report the \emph{attack AUC} (area under curve),
which we introduce as follows.
Fix a set of samples $(g_1,y_1),\cdots,(g_k,y_k)$.
Let the output of some attack be $(s_1,\cdots,s_k)$.
For different threshold $t$,
we can turn the fractional output of the attack into Boolean decision, and obtain corresponding \textit{False  Positive Rate} (FPR) and \textit{True Positive Rate} (TPR) by
\begin{align*}
    FPR(t):=&~\frac{|i\in[k]:s_i\geq t,y_i=0|}{|i\in[k]:y_i=0|},\\
    TPR(t):=&~\frac{|i\in[k]:s_i\geq t ,y_i=1|}{|i\in[k]:y_i=1|}.
\end{align*}
Then the attack AUC is defined the area under the ROC(receiver operating characteristic) curve,
which is
\begin{align*}
    \text{attack AUC}:=\int_{+\infty}^{-\infty} TPR(t) d FPR(t).
\end{align*}
\section{Proofs}\label{sec:missing_proofs}
In this section we present proofs that are omitted in the main paper.

\subsection{Proof of Lemma~\ref{lem:laplace_perturb}}
\begin{proof}[Proof of Lemma~\ref{lem:laplace_perturb}]
Fix $g_0,g_1$.
We start with an auxiliary statement. 
Let the universe $\mathcal{X}\subset \mathbb{R}$ be $\{0,1\}$. 
Consider the identical mapping $f(x):=x$ and the Laplace mechanism $f^{DP}=f+r$ with $r\sim \mathsf{Lap}(b)$.
Then the claim is that 
when $b\geq \frac{1}{\epsilon}$,
$f^{DP}$ is $(\epsilon,0)$-DP.
This is because the $\ell_1$ sensitive of $f$ is $1$,
hence the DP guarantee follows from the property of Laplace mechanism~\citep{dr14}.

Now we consider the deterministic mapping $h:\mathbb{R}\rightarrow \mathbb{R}^d$ defined as
\begin{align*}
    h(x)= x\cdot g_1 + (1-x)\cdot g_0.
\end{align*}
Since $f^{DP}$ is $(\epsilon,0)$-DP,
and DP is immune to post processing,
$h(f^{DP})$ is also $(\epsilon,0)$-DP.
On the other hand, we have
\begin{align*}
    h(f^{DP}(0))= & ~g_0+r\cdot(g_1-g_0),\\
    h(f^{DP}(1))= & ~g_1-r\cdot(g_0-g_1).
\end{align*}
Since $r$ is symmetric, $-r$ is distributed identically to $r$.
Hence we conclude that $h(f^{DP})$ is the same as $\textsf{GradPerturb}$,
which completes the proof of the lemma.
\end{proof}

\subsection{Proof of Lemma~\ref{lem:discrete_perturb}}
\begin{proof}[Proof of Lemma~\ref{lem:discrete_perturb}]
When $U=\mathsf{Bern}(p)$,
we have $\Pr[\textsf{GradPerturb}(y,g_0,g_1)=g_y]=1-p$ and $\Pr[\textsf{GradPerturb}(y,g_0,g_1)=g_{1-y}]=p$.
Hence for all $g\in \{g_0,g_1\}$ and $y\in \{0,1\}$,
we have
\begin{align*}
    \frac{\Pr[\textsf{GradPerturb}(y,g_0,g_1)=g]}{\Pr[\textsf{GradPerturb}(1-y,g_0,g_1)=g]}\leq \frac{1-p}{p},
\end{align*}
which follows from $p\leq \frac{1}{2}$.
This means $\textsf{GradPerturb}$ is $(\ln \frac{1-p}{p},0)$-DP.
\end{proof}

\subsection{Proof of Lemma~\ref{lem:laplace_perturb_multi}}
\begin{proof}[Proof of Lemma~\ref{lem:laplace_perturb_multi}]
We first prove an auxiliary statement: suppose that we have the universe $\mathcal{X}:=\{e_i\}_{i=0}^{k-1}\subset \mathbb{R}^{k}$ and the identity mapping $f(x)=x$,
then the randomized mechanism $f^{DP}:=f+(r_0,r_1,\cdots,r_{k-1})$ is $(\epsilon,0)$-DP where $r_i$ are i.i.d. sampled from $\mathsf{Lap}(2/\epsilon)$. 
This is because the $\ell_1$ sensitivity of $f$ is $2$.

We then consider the deterministic mapping $h:\mathbb{R}^{k}\rightarrow \mathbb{R}^d$ given by
\begin{align*}
    h(x_0,x_1,\cdots,x_{k-1}) = \sum_{i=0}^{k-1} x_ig_i.
\end{align*}
It is easy to verify that $h(f^{DP})$ is just the output of $\textsf{GradPerturb}$.
Hence by the post-processing property of differential privacy,
we have that $\textsf{GradPerturb}$ is $(\epsilon,0)$-DP.
\end{proof}

\subsection{Proof of Lemma~\ref{lem:discrete_perturb_multi}}
\begin{proof}[Proof of Lemma~\ref{lem:discrete_perturb_multi}]
With multi-class discrete perturbation,
the output of \textsf{GradPerturb} is $g_{\tilde y}$ for $\tilde y \in \{0,1,\cdots,k-1\}$.
Furthermore,
for any $y,\tilde y\in \{0,1,\cdots,k-1\}$ we have
\begin{align*}
    \frac{1}{e^{\epsilon}+k-1}\leq \Pr[\textsf{GradPerturb}(y,g_0,g_1,\cdots,g_{k-1})=g_{\tilde y}]\leq \frac{e^{\epsilon}}{e^{\epsilon}+k-1}.
\end{align*}
Hence,
for $y\neq y'$,
\begin{align*}
    \frac{\Pr[\textsf{GradPerturb}(y,g_0,g_1,\cdots,g_{k-1})=g_{\tilde y}]}{\Pr[\textsf{GradPerturb}(y',g_0,g_1,\cdots,g_{k-1})=g_{\tilde y}]}\leq e^{\epsilon},
\end{align*}
which completes the proof of the lemma.
\end{proof}

\end{document}